\pdfoutput=1
\relax
\documentclass[letterpaper]{article}
\usepackage{aaai16_mod}
\usepackage{times}
\usepackage{helvet}
\usepackage{courier}
\usepackage{graphicx}
\usepackage{siunitx}
\usepackage{amsmath,amssymb,amsthm,bm}
\usepackage[norelsize, ruled, lined, boxed]{algorithm2e}
\usepackage[letterpaper]{geometry}
\theoremstyle{plain}
\newtheorem{thm}{Theorem}
\newtheorem{lem}{Lemma}

\theoremstyle{definition}
\newtheorem{defn}{Definition}
\newtheorem{expl}{Example}

\usepackage{color}

\providecommand{\m}[1]{\mathbf{#1}} 

\DeclareMathOperator*{\argmax}{arg\,max}
\newcommand{\norm}[1]{\left\lVert#1\right\rVert}

\makeatletter
\newcommand{\removelatexerror}{\let\@latex@error\@gobble}
\makeatother

\begin{document}

\title{Exploiting Anonymity in Approximate Linear Programming: \\Scaling to Large Multiagent MDPs\\(Extended Version)}

\date{}

\author{Philipp Robbel\\
MIT Media Lab\\
Cambridge, MA, USA\\
\and
Frans A. Oliehoek\\
University of Amsterdam\\
University of Liverpool\\
\and
Mykel J. Kochenderfer\\
Stanford University\\
Stanford, CA, USA}
\maketitle

\begin{abstract} 
\begin{quote}
Many exact and approximate solution methods for Markov Decision Processes (MDPs) exploit structure in the problem and are based on factorization of the value function.  Especially multiagent settings, however, are known to suffer from an exponential increase in value component sizes as interactions become denser, 
meaning that approximation architectures are restricted in the problem sizes and types they can handle.
We present an approach to mitigate this limitation for certain types of multiagent systems, exploiting a property that can be thought of as ``anonymous influence'' in the factored MDP.  Anonymous influence summarizes joint variable effects efficiently whenever the explicit representation of variable identity in the problem can be avoided.
We show how representational benefits from anonymity translate into computational efficiencies, both for 
variable elimination in a factor graph and 
for the approximate linear programming solution to factored MDPs.  
Our methods scale to factored MDPs that were previously unsolvable, such as 
the control of a stochastic disease process over densely connected graphs with 50 nodes and 25 agents.

\end{quote}
\end{abstract}

\section{Introduction}
Cooperative multiagent systems (MASs) present an important framework for modeling the interaction between agents that collaborate to solve a task.  In the decision-theoretic community, 
models like the Markov Decision Process (MDP) and its partially observable extensions have seen widespread use to model and solve such complex planning problems for single and multiple agents in stochastic worlds.  
Multiagent settings, however, are known to suffer from negative complexity results as they scale to realistic settings \cite{Boutilier96:TARK}.  This is because state and action spaces tend to grow exponentially with the agent number, making common solution methods that rely on the full enumeration of the joint spaces prohibitive.

Many problem representations thus attempt to exploit structure in the domain to improve efficiency.  
Factored MDPs (FMDPs) represent the problem in terms of a state space $\mathcal{S}$ that is spanned by a number of state variables, or factors, $X_1,\ldots,X_N$ \cite{Boutilier+al99:JAIR}.  Their multiagent extension (FMMDP) exploits a similar decomposition over the action space $\mathcal{A}$ and allows the direct representation of the ``locality of interaction'' that commonly arises in many multiagent settings \cite{Guestrin+al02:NIPS}.  

Unfortunately, the representational benefits from factored descriptions do not in general translate into gains for policy computation \cite{Koller+Parr99:IJCAI}.  
Even in extremely decentralized settings, for example
Dec-MDPs~\cite{Bernstein+al02:OR}, the value function is coupled because the 
actions of any agent will affect the rewards received in distant
parts of the system.  In FMMDPs this coupling is exacerbated by the fact that
each agent in principle should condition its action on the entire
state~\cite{Oliehoek13AAMAS}.
Still, many solution methods successfully exploit structure in the domain, both in exact 
and approximate settings
, and have demonstrated scalability to large state spaces \cite{Hoey+al99:UAI,Raghavan+al12:AAAI,Cui+al15:AAAI}.  

In this paper we focus on approaches that additionally address larger numbers of agents through value factorization, assuming that smaller, 
\emph{localized} value function components can approximate the global value function well \cite{Guestrin+al03:JAIR,Kok+Vlassis06:JMLR}.  The approximate linear programming (ALP) approach of \citeauthor{Guestrin+al03:JAIR} (\citeyear{Guestrin+al03:JAIR}) is one of the few approaches in this class that retains no exponential dependencies in the number of agents and variables 
through the efficient computation of the constraints in the linear program based on a variable elimination (VE) method.  While the approach improved scalability dramatically, 
the method retains an exponential dependency on the induced tree-width (the size of the largest intermediate term formed during VE), meaning that its feasibility 
depends fundamentally on the connectivity and scale of the factor graph defined by the FMMDP and chosen basis function coverage.

We present an approach that aims to mitigate 
the exponential dependency of VE (both in space and time) 
on the induced width, which is caused by the need to 
represent 
all combinations of state and action variables that appear in each manipulated factor.  In many domains, however, different combinations lead to similar effects, or \emph{influence}.  
Serving as a running example is a disease control scenario over large graphs consisting of uncontrolled and controlled nodes, along with the connections that define possible disease propagation paths 
\cite{Ho+al15:CDC,Cheng+al13:NIPS}.  In this setting the aggregate infection rate of the parent nodes, independent of their individual identity, fully defines the behavior of the propagation model. 
This observation extends to many MASs that are more broadly concerned with the control of dynamic processes on networks, e.g. with stochastic fire propagation models or energy distribution in power grids \cite{Liu+al11:NAT,Cornelius+al13:NAT}.  

We propose to exploit this anonymity of influences for more efficient solution methods for MMDPs.  Our particular contributions are as follows:

\begin{enumerate}
\item We introduce a novel \emph{redundant representation (RR)} for the factors that VE manipulates which involves count aggregators.  This representation is exponentially more compact then regular flat representations and can also be exponentially more compact than existing ``shattered'' representations \cite{Taghipour+alJAIR13,Milch+al08:AAAI}.  
\item We show how to derive an efficient VE algorithm, RR-VE, that makes use of the redundant representation, and prove its correctness.  While the induced tree width does not change, since operations on the factors modified during VE can now run in less than exponential time with respect to the number of variables, RR-VE can scale theoretically to much larger problems. 
\item We then propose RR-ALP, which extends the ALP approach by making use of RR-VE, and maintains identical solutions.  The RR-ALP consists of equivalent but smaller constraints sets for factored MDPs that support anonymous influence.  
\item We show an empirical evaluation of our methods that demonstrates speed-ups of the ALP method by an order of magnitude in a sampled set of random disease propagation graphs with 30 nodes.  We also demonstrate the ability to scale to problem sizes that were previously infeasible to solve with the ALP solution method and show how the obtained policy outperforms a hand-crafted heuristic by a wide margin in a 50-node disease control problem with 25 agents.
\end{enumerate}

\section{Background}
We first discuss the necessary background on factored multiagent MDPs and their efficient solution methods that are based on value factorization.

\subsection{Factored Multiagent MDPs}
\label{sec:2}
Markov decision processes are a general framework for decision making under uncertainty \cite{Kochenderfer15:DMU,Puterman05:MDP}.  An infinite-horizon Markov decision process (MDP) is defined by the tuple $\langle \mathcal{S},\mathcal{A},T,R,\gamma \rangle$, where $\mathcal{S}=\{s_1,\ldots,s_{|\mathcal{S}|}\}$ and $\mathcal{A}=\{a_1,\ldots,a_{|\mathcal{A}|}\}$ are the finite sets of states and actions, $T$ the transition probability function specifying $P(s'\mid s,a)$, $R(s,a)$ the immediate reward function, and $\gamma\in [0,1]$ the discount factor of the problem.
\begin{figure}[t]
  \begin{center}
    \includegraphics[width=0.26\columnwidth]{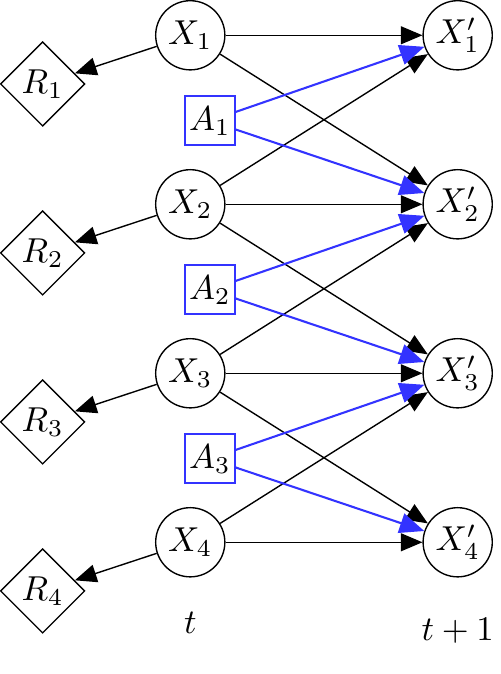}
  \end{center}
  \caption{A two-slice temporal Bayesian network (2TBN) representation of the collaborative factored multiagent MDP (FMMDP) with three agents, where parent node sets $\textnormal{Pa}(X_i')$ include both state and action variables.}
  \label{fig:2tbn}
\end{figure}

Factored MDPs (FMDPs) exploit structure in the state space $\mathcal{S}$ and define the system state by an assignment to the state variables $\mathbf{X} = \{X_1,\ldots,X_n\}$.  Transition and reward function decompose into a two-slice temporal Bayesian network (2TBN) consisting of independent factors, each described by their scope-restricted conditional probability distributions (CPDs) \cite{Boutilier+al99:JAIR}.

In the case of collaborative multiagent systems, the agent set $\mathbf{A} = \{A_1,\ldots,A_g\}$ additionally spans a joint action space $\mathcal{A}$ that is generally exponential in the number of agents.  The factored multiagent MDP (FMMDP) is a tractable representation that introduces action variables into the 2TBN \cite{Guestrin+al02:NIPS} (see Figure~\ref{fig:2tbn} for an illustration).  
The FMMDP transition function can be written as
\begin{equation}
\label{eqn:2tbn}
P(\mathbf{x'}\mid \mathbf{x},\mathbf{a}) = \prod_i 
T_i(x_i'\mid \mathbf{x}[\textnormal{Pa}(X_i')],\mathbf{a}[\textnormal{Pa}(X_i')])
\end{equation}
where $\textnormal{Pa}(X_i')$ refers to the parent nodes of $X_i'$ in the 2TBN (covering both state \emph{and} action variables), 
and $\mathbf{x}[\textnormal{Pa}(X_i')]$ to their value in state $\mathbf{x}$.  The value of the respective action variables is analogously denoted by $\mathbf{a}[\textnormal{Pa}(X_i')]$.  Collaborative FMMDPs assume that each agent $i$ observes part of the global reward and is associated with (restricted scope) 
local reward function $R_i$, such that the global reward factors additively as $R(\mathbf{x},\mathbf{a}) = \sum_{i=1}^g R_i(\mathbf{x}[\m{C}_i],\mathbf{a}[\m{D}_i])$ for some subsets of state and action variables $\m{C}_i$ and $\m{D}_i$, respectively.  

The solution to an (M)MDP is a (joint) policy that optimizes some optimality criterion about the rewards, e.g. 
the expected sum of discounted rewards that 
can be achieved from any state.  
We consider value-based solution methods that store this expected return for every state $\m{x}$ in a (state) value function $\mathcal{V}(\m{x})$.  
The optimal value function $\mathcal{V}^*(\m{x})$ represents the maximum
expected return possible from every state \cite{Puterman05:MDP}.
Such an (optimal) value function can be used to extract an (optimal) policy by
performing a \emph{back-projection} through the transition function to compute the so-called Q-function:
\begin{equation}
\label{eqn:q-fn}
    \forall_{\m{x},\m{a}} \quad
\mathcal{Q}^*(\m{x},\m{a}) = R(\m{x},\m{a}) + \gamma \sum_{\m{x'}}P(\m{x}'\mid
\m{x},\m{a}) \mathcal{V}^*(\m{x}),
\end{equation}
and subsequently acting greedy with respect to the Q-function: the
optimal action at $\m{x}$ is $\m{a}^* = \arg \max \mathcal{Q}^*(\m{x},\m{a}) $.

\subsection{Control of Epidemics on Graphs}
\label{sec:2-dom}
We use the control of a disease outbreak over a graph as a running example from the wider problem class of controlling dynamic processes on networks \cite{Ho+al15:CDC}.  
%
The disease outbreak dynamics follow a version of the susceptible-infected-susceptible (SIS) model with homogeneous model parameters in the network \cite{Bailey57:book}.  SIS dynamics have been thoroughly studied for many years but only few approaches consider the complex \emph{control} problem of epidemic processes \cite{Nowzari+al15:arXiv,Ho+al15:CDC}.

SIS dynamics with homogeneous parameters are modeled as an FMMDP as follows.  We define the network as a (directed or undirected) graph $G=(V,E)$ with controlled and uncontrolled vertices $V=(V_c,V_u)$ and edge set $E\subseteq V\times V$.  The state space $\mathcal{S}$ is spanned by state variables $X_1,\ldots,X_n$, one per associated vertex $V_i$, encoding the health of that node.  The agent set $\m{A}=\{A_1,\ldots,A_{|V_c|}\}$ factors similarly over the controlled vertices $V_c$ in the graph and denote an active modulation of the flow out of node $V_i\in V_c$.  Note that this model assumes binary state variables $X_i=\{0,1\}=\{\text{healthy},\text{infected}\}$, and actions $A_i=\{0,1\}=\{\text{do not vaccinate},\text{vaccinate}\}$ and that $A_u=\{0\}$ for all uncontrolled nodes $V_u$.  

Let $x_i$ and $a_i$ denote the state and action for a single node.  
The transition function factors on a per-node basis into 
$T_i(x_i' = \text{infected} \mid \mathbf{x}[\textnormal{Pa}(X_i')],a_i)$ defined as:
\begin{equation}
T_i \triangleq
\begin{cases}
(1-a_i)(1-\prod_j(1-\beta_{ji}x_j)) & \text{if }x_i=0\\
(1-a_i)\,(1-\delta_i) & \text{otherwise}\\
\end{cases}
\end{equation}
distinguishing the two cases that $X_i$ was infected at the previous time step (bottom) or not (top).  Parameters $\beta_{ji}$ and $\delta_i$ are the known infection transmission probabilities from node $j$~to~$i$, and node $i$'s recovery rate, respectively.  The reward function factors as:
\begin{equation}
R(\mathbf{x},\mathbf{a}) = -\lambda_1\norm{\mathbf{a}}_1 -\lambda_2\norm{\mathbf{x}}_1
\end{equation}
where the $L_1$ norm records a cost $\lambda_2$ per infected node 
and an action cost $\lambda_1$ per vaccination action at a controlled node.

\subsection{Efficient Solution of Large FMMDPs}
\label{sec:effsoln}
We build upon the work by \citeauthor{Guestrin+al03:JAIR} (\citeyear{Guestrin+al03:JAIR}), who present approximate solution methods for FMMDPs that are particularly scalable in terms of the agent number.  The approach represents the global Q-value function as the sum of appropriately chosen smaller value function components, each defined over a subset of state and action variables.  Referred to as factored linear value functions, they permit efficient operations, such as the computation of the jointly maximizing action in a state, by treating it as a constraint optimization problem (e.g., \cite{Dechter13:book}) where the maximizing configuration can be found with methods such as variable elimination (VE) that avoid the enumeration of exponentially many actions.

To compute a factored linear value function, 
the same authors present an efficient method to compactly represent the constraints in the approximate linear programming (ALP) solution to MDPs. 
Their extensions circumvent exponentially many constraints by implicitly making use of the above VE trick with factored linear value functions: since the constraints in the linear program can be interpreted as a maximization over sums of local terms, it is possible to replace them by an equivalent, smaller set of constraints based on the same insights as above.  
The remainder of this subsection gives more detail on this efficient solution method.

\subsubsection{Factored Value Functions}
Value factorization is one successful approach that addresses both large $\mathcal{S}$ and $\mathcal{A}$ by representing the joint value function as a linear combination of locally-scoped terms.  Each local term applies to a part of the system and covers potentially multiple, even overlapping, state factors: $\mathcal{V}(\m{x}) = \sum_i\mathcal{V}_i(\m{x}[\m{C}_i])$ for local state scopes $\mathbf{C}_i \subseteq \{X_1,\ldots,X_n\}$.
Note that in the limit of a single value function term this representation is simply a single joint value function in the global state $\m{x}$; still, one may hope that a set of lower-dimensional components may yield an adequate approximation in a large structured system. 

In the case of factored \emph{linear} value functions given a set of (possibly non-linear)
basis functions $H=\{h_1,\ldots,h_k\}$, 
$\mathcal{V}$ can be written as the linear combination 
$\mathcal{V}(\m{x}) = \sum_{j=1}^k w_j h_j(\m{x})$ 
where 
$h_j$ is defined over some 
subset of variables $\m{C}_{h_j}\subseteq \m{X}$ (omitted for clarity), and $w_j$ is the weight associated with basis $h_j$.  

Factored linear (state) value functions induce factored Q-value functions if transitions and rewards are factored into local terms.  
In this case, the \emph{back-projection} in the computation of the Q-value function can be computed efficiently by avoiding the sum over exponentially many successor states in Equation~\ref{eqn:q-fn} \cite{Guestrin03:PhD}.  
This is because 
the expectation over an individual basis function $h_j(\m{x}[\m{C}_{h_j}])$ 
can be computed efficiently since the scope of the variables that appear as \emph{parents} of $X_i\in\m{C}_{h_j}$ in the 2TBN remains local.  These expectations are referred to as \emph{basis back-projections} of the functions $h_j$ and denoted by $g_j(\m{x},\m{a})$.  

\begin{defn}[Basis back-projection]
\label{defn:backprop}
Given a basis function $h_j:\m{C}\to\mathbb{R}$, defined over scope $\m{C}\subseteq \m{X}$, and a factored 2TBN transition model $P(\m{x'}\mid \m{x},\m{a})$ (see Equation~\ref{eqn:2tbn}), define the \emph{basis back-projection} of $h_j$ as:
\begin{equation}
\label{eqn:backprop}
\begin{array}{rcl}
g_j(\m{x},\m{a}) & \triangleq & \sum_{\m{x'}}P(\m{x'}\mid \m{x},\m{a})\,h_j(\m{x'}[\m{C}])\\
& = & \sum_{\m{c}'} P(\m{c}'\mid \m{x},\m{a})\,h_j(\m{c}')\\
& = & \sum_{\m{c}'} P(\m{c}'\mid \m{x}[\textnormal{Pa}(\m{C})],\m{a}[\textnormal{Pa}(\m{C})])\,h_j(\m{c}')\\
\end{array}
\end{equation}
where $\textnormal{Pa}(\m{C})\triangleq \bigcup_{X_i\in\m{C}} \textnormal{Pa}(X_i)$ denotes the union of respective parent (state and action) variables in the 2TBN.
\end{defn}
Functions $g_j$ are thus again locally-scoped, defined precisely over the parent scope $\textnormal{Pa}(\m{C})$ (omitted for clarity in the remainder of the presentation).  Basis back-projections are used to compute a factored Q-value function:
\begin{equation}
\label{eqn:fac-q}
\begin{array}{rcl}
    \mathcal{Q}(\mathbf{x},\mathbf{a}) &=& R(\mathbf{x},\mathbf{a}) + \gamma \sum_{\mathbf{x'}}P(\mathbf{x'}|\mathbf{x},\mathbf{a}) \sum_jw_jh_{j}(\mathbf{x'})\\
    &=& \sum_{r} R_r(\mathbf{x}[\m{C}_r],\mathbf{a}[\m{D}_r]) + \gamma \sum_{j} w_j g_j(\m{x},\m{a})\\
    &=& \sum_{i} \mathcal{Q}_i(\m{x}[\m{C}_i],\m{a}[\m{D}_i])
\end{array}
\end{equation}
where the last line in Equation~\ref{eqn:fac-q} follows by associating disjoint subsets of local reward functions and basis back-projections with each $\mathcal{Q}_i$.  The factor graph spanned by a factored Q-value function instantiated in a particular state $\m{x}$ is in this context often referred to as a \emph{coordination graph} (CG).

\subsubsection{VE} 
\providecommand{\Aug}{\textsc{Augment}}
\providecommand{\Red}{\textsc{Reduce}}

The variable elimination (VE) algorithm can be used for computing the max over a set of locally-scoped functions in a factor graph efficiently.  Similarly to maximum a posteriori (MAP) estimation in Bayesian networks, VE maximizes over single variables at a time rather than enumerating all possible joint configurations followed by picking the maximizing one \cite{Koller+Friedman09:PGM}.  

Variable elimination performs two operations, $\Aug$ and $\Red$, repeatedly for every variable $X_{l}$ to be eliminated from the factor graph.  
Here, $\Aug$ corresponds to the sum of functions that depend on $X_{l}$ and $\Red$ to the maximization over $X_{l}$ in the result (see Figure \ref{fig:alg-ve}).
The execution time is exponential in the size of the largest intermediate term formed which depends on the chosen elimination order.  While the problem of determining the optimal elimination order is NP-complete, effective heuristics for variable ordering exist in practice \cite{Koller+Friedman09:PGM}.

\begin{figure}[htp]
\removelatexerror
\scalebox{.85}{ 
\begin{algorithm}[H]
\SetAlgoLined
\KwIn{$\mathcal{F}$ is a set of functions}
\KwIn{$\mathcal{O}$ is the elimination order over all variables}
\KwOut{The result of the maximization over all variables referred by $\mathcal{O}$}

\For{$i=1,\ldots,|\mathcal{O}|$}{
$l = \mathcal{O}(i)$\;
\tcp{Collect functions that depend on $X_l$}
$\mathcal{E} = \textsc{Collect}(\mathcal{F},X_l)$\;
\tcp{Compute the sum}
$f = \Aug(\mathcal{E})$\;
\tcp{Compute the max}
$e = \Red(f,X_l)$\;
\tcp{Update the function set}
$\mathcal{F} = \mathcal{F}\cup{\{e\}}\setminus \mathcal{E}$\;
}
\tcp{Sum the empty-scope functions}
\Return{$\Aug(\mathcal{F})$}\;
\caption{$\textsc{VariableElimination}(\mathcal{F},\mathcal{O})$}
\end{algorithm}
}
\caption{The \textsc{VariableElimination} algorithm computing the maximum value of $\sum_{f\in\mathcal{F}}f$ over the state space.}
\label{fig:alg-ve}
\end{figure}

VE also finds application in computing the maximizing joint action in a coordination graph defined over locally-scoped Q-value function terms, i.e.,
\[
\m{a}^*=\argmax_{\m{a}} \sum_i \mathcal{Q}_i(\m{x}[\m{C}_i],\m{a}[\m{D}_i])
\]
can be done efficiently with the decision-making equivalent to VE in a Bayesian network \cite{Guestrin+al02:NIPS,Kok+Vlassis06:JMLR}.  As a result, action selection in a particular state $\m{x}$ can avoid the direct enumeration of exponentially many (joint) action choices.

\subsubsection{ALP} 
VE can be used for efficient joint action selection in a particular state given a factored Q-function, but it does not give a way to directly compute such a factored Q-function.  The approximate linear programming (ALP) variant introduced by \cite{Guestrin+al03:JAIR} does allow this by implicitly making use of the above VE technique.  It builds on the regular ALP method for solving MDPs which computes the best approximation (in a weighted $L_1$ norm sense) to the optimal value function in the space spanned by the basis functions \cite{Puterman05:MDP}.  The basic ALP formulation for an infinite horizon discounted MDP given basis choice $h_1,\ldots,h_k$ is given by:
\begin{equation}
\label{alp}
\begin{array}{rl}
\displaystyle \min_{\m{w}} & \sum_{\m{x}} \alpha(\m{x}) \sum_i w_i h_i(\m{x}) \\
\textrm{s.t.} & \sum_i w_i h_i(\m{x}) \geq [R(\m{x},\m{a}) + \gamma \sum_{\m{x'}}P(\m{x'}\mid \m{x},\m{a})\sum_iw_ih_i(\m{x'})]\,\forall \m{x},\m{a}
\end{array}
\end{equation}
for state relevance weights $\alpha(\m{x})$ (assumed uniform here) and variables $w_i$ unbounded.  The ALP yields a solution in time polynomial in the sizes of $\mathcal{S}$ and $\mathcal{A}$ but these are exponential for MASs.

\citeauthor{Guestrin03:PhD} (\citeyear{Guestrin03:PhD}) introduces an efficient implementation of the ALP for factored linear value functions that avoids the exponentially many constraints in the ALP.  It applies if the basis functions have \emph{local scope} and transitions and rewards are factored.  Underlying it are two insights:

First, the sum over exponentially many successor states $\m{x'}$ \emph{in the constraints} in Equation~\ref{alp} can be avoided by realizing that the right-hand side of the constraints corresponds to the (factored) Q-function that was previously shown to admit efficient computation via basis back-projections (Definition~\ref{defn:backprop}).  

The second insight is that all (exponentially many) constraints in the ALP can 
be reformulated as follows:
\begin{equation}
\label{eqn:guestrin}
\setlength{\arraycolsep}{0.15cm}
\begin{array}{rrrcl}
&\forall \m{x},\m{a} & \sum_i w_i h_i(\m{x}) & \geq & R(\m{x},\m{a}) + \gamma \sum_i w_i\, g_i(\m{x},\m{a})\\
\Rightarrow &\forall \m{x},\m{a} & 0 & \geq & R(\m{x},\m{a}) + \sum_i w_i [\gamma g_i(\m{x},\m{a}) - h_i(\m{x})]\\
\Rightarrow& & 0 & \geq & \max_{\m{x},\m{a}}[\sum_rR_r(\m{x}[\m{C}_r],\m{a}[\m{D}_r])\; + \sum_i w_i [\gamma g_i(\m{x},\m{a}) - h_i(\m{x})]]\\
\end{array}
\end{equation}
The reformulation replaces the exponential set of linear constraints with a \emph{single} non-linear constraint (last row in Equation~\ref{eqn:guestrin}).  Using a procedure similar to VE, this max constraint can be implemented with a small set of linear constraints, avoiding the enumeration of the exponential state and action spaces.  To see this, consider an arbitrary intermediate term obtained during VE, $e'(\m{x}[\m{C}])=\Red(e(\m{x}[\m{C}\cup \{X_k\}]),X_k)$.  Enforcing that $e'$ is maximal over its domain can be implemented with $|Dom(e')|$ new variables and $|Dom(e)|$ new linear constraints in the ALP \cite{Guestrin03:PhD}:
\begin{equation}
e'(\m{x}[\m{C}]) \geq e(\m{x}[\m{C}\cup \{X_k\}]) \quad\forall \m{x}[\m{C}\cup \{X_k\}]\in Dom(e).
\end{equation}
The total number of linear constraints to implement the max constraint in Equation~\ref{eqn:guestrin} is only exponential in the size of the largest intermediate term formed during VE. 

\section{Anonymous Influence}
\label{sec:3}
At the core of the ALP solution method lies the assumption that VE can be carried out efficiently in the factor graph spanned by the local functions that make up the max constraint of Equation~\ref{eqn:guestrin}, i.e. that the scopes of all intermediate terms during VE remain small.  This assumption is often violated in many graphs of interest, e.g., in disease control where nodes may possess large in- or out-degrees.

In this section we develop a novel approach to deal with larger scope sizes in VE than were previously feasible.  Underlying it is the insight that in the class of graph-based problems considered here, 
only the \emph{joint effects} of sets of variables---rather than their identity---suffices to compactly describe the factors that appear in the max constraint and are manipulated during VE.  
We introduce a novel representation that is exponentially smaller than the equivalent full encoding of intermediate terms and show how VE retains correctness.  

First, we address the \emph{representation} of ``joint effects'' before showing 
how it can be exploited \emph{computationally} during VE and in the ALP.  
In our exposition we assume binary variables but the results carry over to the more general, discrete variable setting.

\subsection{Mixed-Mode Functions}
We define count aggregator functions to summarize the ``anonymous influence'' of a set of variables.  In the disease propagation scenario for example, the number of active parents uniquely defines the transition model $T_i$; the identity of the parent nodes is irrelevant for representing $T_i$.  The following definitions formalize this intuition.

\begin{defn}[Count Aggregator] Let $\mathbf{Z}=\left\{ Z_{1},\ldots,Z_{|\m{Z}|}\right\} $ be a set of binary variables, $Z_{i}\in\{0,1\}$.  The \emph{count aggregator (CA)} $\#\{\mathbf{Z}\}:Z_{1}\times\ldots\times Z_{|\m{Z}|}\mapsto\{0,\dots,|\m{Z}|\}$ is defined as: $\#\{\mathbf{Z}\}(\m{z})\triangleq\sum_{i=1}^{|\m{Z}|}z_{i}$.  $\m{Z}$ is also referred to as the \emph{count scope} of CA $\#\{\m{Z}\}$.
\end{defn}
Hence, CAs simply summarize the number of variables that appear `enabled' in its domain.  Conceptual similarities with generalized (or `lifted') counters in first-order inference are discussed in Section~\ref{sec:7}.  Functions that rely on CAs can be represented compactly.  

\begin{defn}[Count Aggregator Function] A \emph{count aggregator function (CAF)}, is a function $f:\mathbf{Z}\to\mathbb{R}$ that maps $\mathbf{Z}$ to the reals by making use of a CA.  That is, there exists a function $\mathfrak{f}:\left\{ 0,\dots,\left|\mathbf{Z}\right|\right\} \to\mathbb{R}$ such that $f$ can be defined with the function composition operator as:
\begin{equation}
f(\m{z})\triangleq\left[\mathfrak{f}\circ\#\{\mathbf{Z}\}\right](\mathbf{z}).\label{eq:CAF-definition}
\end{equation}
To make clear $f$'s use of a CA, we use the notation $f(\#(\mathbf{z}))$.
\end{defn}

CAFs have a \emph{compact representation} which is precisely the function $\mathfrak{f}$. It is compact, since it can be represented using $\left|\mathbf{Z}\right|+1$ numbers and $\left|\mathbf{Z}\right|+1\ll2^{\left|\mathbf{Z}\right|}$. 
Generally, whenever function representations are explicitly referred to in this paper, the fractal font is used.

We now introduce so-called ``mixed-mode'' functions $f$ that depend both on CAs and on other variables $\mathbf{X}$ that are not part of any CA:

\begin{defn}[Mixed-Mode Function] \label{def:mmfi} A function $f:\mathbf{X}\times\mathbf{Z}\rightarrow\mathbb{R}$
is called a \emph{mixed-mode function} (MMF), denoted $f(\mathbf{x},\#(\mathbf{z}))$,
if and only if
$
\forall\mathbf{x}\ \exists f_{\mathbf{X}}~ s.t.~ f(\mathbf{x},\mathbf{z})=f_{\mathbf{X}}(\#(\mathbf{z})).
$
That is, for each instantiation $\mathbf{x}$, there exists a CAF $f_{\mathbf{X}}(\#(\m{z}))$. We refer to $X_{i}\in\mathbf{X}$
as \emph{proper variables} and $Z_{j}\in\mathbf{Z}$ as \emph{count
variables} in the scope of~$f$. \end{defn}

\begin{expl}
\label{expl:dprop}
Consider the conditional probability distribution 
$T_i(X_i\mid \textnormal{Pa}(X_i))$ of a (binary) node $X_i$ and its parents 
in the (binary) disease propagation graph.  Let $x_i$ and $\bar{x}_i$ denote the case that node $i$ is infected and not infected, respectively.  Then $T_i(X_i\mid \#\{\textnormal{Pa}(X_i)\})$ is a \emph{mixed-mode function} that induces two CAFs, one for $x_i$ and one for $\bar{x}_i$.
\end{expl}

Mixed-mode functions generalize simply to those with multiple CAs, $f:\mathbf{X}\times\mathbf{Z}_{1}\times\ldots\times\mathbf{Z}_{N}\rightarrow\mathbb{R}$, denoted $f(\mathbf{x},\#_{1}(\mathbf{z}_{1}),\ldots,\#_{N}(\mathbf{z}_{N}))$.  
The following cases can occur:
\begin{enumerate}
\item 
MMFs with \emph{fully disjoint scopes} have mutually disjoint proper
and count variable sets, i.e., $\mathbf{X}\cap\mathbf{Z}_{i}=\emptyset\;\forall i=1,\ldots,N$
and $\mathbf{Z}_{i}\cap\mathbf{Z}_{j}=\emptyset\;\forall i\neq j$;
\item 
MMFs have \emph{shared proper and count variables} if and only if
$\exists i\;s.t.\;\mathbf{X}\cap\mathbf{Z}_{i}\neq\emptyset$;
\item 
MMFs have \emph{non-disjoint counter scopes} if and only if $\exists(i,j),i\neq j\;s.t.\;\mathbf{Z}_{i}\cap\mathbf{Z}_{j}\neq\emptyset$.
\end{enumerate}
In our treatment of MMFs we often refer to the \emph{canonical notation} $f(x,y,z,\#_{1}(a,b,z),\#_{2}(b,c))$ to denote a general MMF that includes both shared proper and count variables, as well as non-disjoint counter scopes. 

Summarizing, it is possible to represent certain anonymous influences using mixed-mode functions. In the following we will show that these can be compactly represented, which subsequently forms the basis for a more efficient VE algorithm.

\subsection{Compact Representation of MMFs}
Just as CAFs, a mixed-mode function $f$ has a compact \emph{representation} $\mathfrak{f}:\m{X}\times\{0,...,|\m{Z}|\}\rightarrow\mathbb{R}$ where 
$
f(\m{x},\#(\m{z}))\triangleq\mathfrak{f}(\m{x},\sum_{i=1}^{|\m{Z}|}z_{i}).
$
A mixed-mode function $f$ can thus be described with (at most) $K^{|\m{X}|}(|\m{Z}|+1)$ parameters where $K$ is an upper bound on $|Dom(X_{i})|$.

As mentioned before, we also consider MMFs with multiple CAs. In particular, let us examine a function $f(\#_{1}(a,b),\#_{2}(b,c))$ with two CAs that have a overlapping scope since both depend on shared variable $B$.
In order to consistently deal with overlaps in the count scope, previous
work has considered so-called shattered representations \cite{Taghipour+alJAIR13,Milch+al08:AAAI}.
A MMF with overlapping count scopes $f(\#_{1}(a,b),\#_{2}(b,c))$ can always be transformed into an equivalent one without overlapping count scopes $f'(\#_{1}'(a),\#_{2}'(c),\#(b))$ by defining a new function that is equivalent to it:
\[
f'(\#_{1}'(a),\#_{2}'(c),\#(b))\triangleq f(\#_{1}(a,b),\#_{2}(b,c)).
\]

We can now distinguish between different \emph{representations} of these MMFs with overlapping count scopes.

\begin{defn}[Shattered Representation]
The \emph{shattered representation} of $f$ is the 
representation of $f'$, i.e.
\[
f(\#_{1}(a,b),\#_{2}(b,c))\triangleq\mathfrak{f}(k_{1},k_{2},k_{3})
\]
where $k_{1}:=a$, $k_{2}:=c$, $k_{3}:=b$ and $\mathfrak{f}:\{0,1\}\times\{0,1\}\times\{0,1\}\rightarrow\mathbb{R}$.
\end{defn}

We introduce a novel \emph{redundant representation} of $f$.  Redundant representations retain compactness with many overlapping count scopes.  This becomes relevant when we introduce operations on MMFs (e.g., for variable elimination) in later sections of the paper.

\begin{defn}[Redundant Representation] The redundant representation
of MMF $f(\#_{1}(a,b),\#_{2}(b,c))$ is a function $\mathfrak{f}:\{0,1,2\}\times\{0,1,2\}\rightarrow\mathbb{R}$:
\[
f(\#_{1}(a,b),\#_{2}(b,c))\triangleq\mathfrak{f}(k_{1},k_{2})
\]
where $k_{1}:=a+b$ and $k_{2}:=b+c$.
\end{defn}

If we choose to store MMFs with redundant representations, we may introduce incompatible assignments to variables that appear in overlapping count scopes.  The following definition formalizes this observation.

\begin{defn}[Consistent Count Combination]
Let $\#_{1}\{A,B\},\#_{2}\{B,C\}$ be two CAs with overlapping count
scopes. We say that a pair $\left(k_{1},k_{2}\right)$ is a \emph{consistent
count combination (consistent CC) for $\#_{1},\#_{2}$ }if and only
if there exists an assignment $\left(a,b,c\right)$ such that 
$\left(k_{1},k_{2}\right)=\left(\#_{1}(a,b),\#_{2}(b,c)\right)$. 
If no such $\left(a,b,c\right)$ exists, then $\left(k_{1},k_{2}\right)$
is called an \emph{inconsistent} CC.  Further, let $f(\#_{1},\#_{2})$
be a MMF. We say that a consistent CC $\left(k_{1},k_{2}\right)$
for $\#_{1},\#_{2}$ is a \emph{consistent entry} $\mathfrak{f}(k_{1},k_{2})$
of the representation of $f$. Similarly, if $\left(k_{1},k_{2}\right)$
is an inconsistent CC, then $\mathfrak{f}(k_{1},k_{2})$ is referred
to as an \emph{inconsistent entry}.
\end{defn}

Inconsistent entries can only occur in redundant representations
since the shattered representation of $f$ is defined for $f'$ without overlapping count scopes.
%
Even though redundant representations appear to have a disadvantage
since they contain inconsistent entries, they also have a big advantage:
as we show next, they can be exponentially more compact than shattered
ones.  Moreover, as detailed in the rest of this document, the disadvantage
of the inconsistent entries can be avoided altogether by making
sure that we never query such inconsistent entries in our algorithms.
\begin{lem}
Consider MMF $f:\mathbf{X}\times\mathbf{Z}_{1}\times\ldots\times\mathbf{Z}_{N}\rightarrow\mathbb{R}$, $N\geq2$. Let $\mathbf{Z=}\bigcup_{i=1}^{N}\mathbf{Z}_{i}$. In the
worst case, a partition of $\mathbf{Z}$ requires $p=\min\{2^{N}-1,|\mathbf{Z}|\}$
splits into mutually disjoint sets and the shattered representation
of $f$ is of size $O(S^{p})$ where $S$ is an upper bound on the
resulting set sizes. The same function has a redundant representation
of size $O(K^{N})$ where $K$ is an upper bound on $|\mathbf{Z}_{i}|+1$.
\end{lem}

\begin{expl} 
Consider MMF $f(\#_{1}\{A,B,C,D,E\},$ $\#_{2}\{A,B,X,Y,Z\},$
$\#_{3}\{A,C,W,X\})$ with overlapping count scopes. The redundant
representation of $f$ requires $6\cdot6\cdot5=180$ parameters but
contains \emph{inconsistent entries}. The shattered representation
defined using equivalent MMF $f'(\#\{A\},\#\{B\},$ $\#\{C\},\#\{D,E\},\#\{X\},$ $\#\{W\},\#\{Y,Z\})$ requires $288$ parameters. 
\end{expl} 
Note that, in general, the difference in size between shattered and un-shattered representations can be made arbitrarily large. 
We now show how mixed-mode functions with compact redundant representations can be exploited during variable elimination and during constraint generation in the ALP.

\section{Efficient Variable Elimination}
\label{sec:4}

Here we describe how $\Aug$ and $\Red$ are efficiently implemented to work \emph{directly} on the redundant representations of MMFs.  Our goal is to leverage compact representation throughout the VE algorithm, i.e., to avoid shattering of function scopes if possible.  
Particular care has to be taken to ensure correctness since we observed previously that reduced representations contain inconsistent entries.

\subsubsection{Augment} 
$\Aug$ takes a set of MMFs and adds them together. We implement this
operation directly in the redundant representation.  $\Aug(\mathfrak{g},\mathfrak{h})$
returns a function $\mathfrak{f}$ that is defined as: $\forall x,y,k_{1}\in\{0,\ldots,N_{1}\},k_{2}\in\{0,\ldots,N_{2}\}$
\begin{equation}
\mathfrak{f}(x,y,k_{1},k_{2})=\mathfrak{g}(x,k_{1})+\mathfrak{h}(y,k_{2}).\label{eq:RR-Augment}
\end{equation}
The implementation simply loops over all $x,y,k_{1},k_{2}$ to compute all entries 
(which may be consistent or inconsistent).

\subsubsection{Reduce}
$\Red$ removes a variable by maxing it out. 
Here we show how this operation is implemented for MMFs directly using the redundant representation.
Let $g(x,y,z,\#_{1}(a,b,z),\#_{2}(b,c))$ be a MMF with redundant representation $\mathfrak{g}(x,y,z,k_{1},k_{2})$. We discriminate different cases:
\begin{enumerate}
\item 

\emph{Maxing out a proper variable:} If we max out $x$, 
$
\mathfrak{f}(y,z,k_{1},k_{2})\triangleq\max\left\{ \mathfrak{g}(0,y,z,k_{1},k_{2}),\mathfrak{g}(1,y,z,k_{1},k_{2})\right\}
$

\item 
\emph{Maxing out a non-shared count variable:}
If we max out $a$,
$
\mathfrak{f}(x,y,z,k_{1},k_{2})\triangleq\max \{ \mathfrak{g}(x,y,z,k_{1},k_{2}),
$
$
\mathfrak{g}(x,y,z,k_{1}+1,k_{2})\}.
$
The resulting function has signature
$f(x,y,z,\#_{1}'(b,z),\#_{2}'(b,c))$.
The values of $x,y,z,b,c$ are fixed (by the l.h.s. of the definition) 
in such a way that $\#_{1}'(b,z)=k_{1}$
and $\#_{2}'(b,c)=k_{2}$. The maximization that we perform over $a\in\{0,1\}$
therefore has the ability to increase $k_{1}$ by 1 or not, which
leads to the above maximization in the redundant representation. 

\item 
\emph{Maxing out a shared count variable:}
If we max out $b$,
$
\mathfrak{f}(x,y,z,k_{1},k_{2})\triangleq\max\{ \mathfrak{g}(x,y,z,k_{1},k_{2}), \mathfrak{g}(x,y,z,$ $k_{1}+1,k_{2}+1)\}
$
This is similar to the previous case, but since $b$ occurs in both $\#_{1}'$ and
$\#_{2}'$, it may either increase both $k_{1}$ and $k_{2}$, or neither.

\item 
\emph{Maxing out a shared proper/count variable:}
In case we max out $z$,
$
\mathfrak{f}(x,y,k_{1},k_{2})\triangleq\max\{ \mathfrak{g}(x,y,0,k_{1},k_{2}),$ $\mathfrak{g}(x,y,1,k_{1}+1,k_{2})\}.
$
Since $z$ occurs as both proper and count variable (in $\#_{1}$),
a choice of $z=1$ also increments $k_{1}$ by 1 while $z=0$ does
not.
\end{enumerate}

\subsubsection{Correctness of RR-VE}
We refer to VE with the elementary operations defined as above as \emph{redundant representation VE (RR-VE)}.  RR-VE is correct, i.e., it arrives at the identical solution as VE using the full tabular representation of intermediate functions.
The proof depends on the following two lemmas:
\begin{lem}\label{lem:rr-ve-aug} When the input functions are correctly defined on their
consistent entries, $\Aug$ is correct.\end{lem}
\begin{proof}The implementation of $\Aug$ as given in \eqref{eq:RR-Augment}
may lead to an $\mathfrak{f}$ that contains inconsistent entries.
In particular, this will happen when the scopes of $\#_{1},\#_{2}$
contain a shared variable. However, we will make sure that when using $f$ later on, we will only query consistent entries. That
is, when querying the function $f(x,y,\#_{1}(a,b),\#_{2}(b,c))$ via
arguments $(x,y,a,b,c)$ we will never retrieve such inconsistent
entries. Therefore, we only need to show that the consistent entries
are computed correctly. In particular, we need to show that 
\begin{align*}
\forall_{x,y,a,b,c}\quad f(x,y,\#_{1}(a,b),\#_{2}(b,c)) & =g(x,\#_{1}(a,b))+h(y,\#_{2}(b,c))\\
 & =\mathfrak{g}(x,a+b)+\mathfrak{h}(y,b+c)\\
\text{\{we define }k_{1}:=a+b,\,k_{2}:=b+c\text{ \}} & =\mathfrak{g}(x,k_{1})+\mathfrak{h}(y,k_{2})
\end{align*}
which is exactly the value that \eqref{eq:RR-Augment} computes for
(consistent) entry $\mathfrak{f}(x,y,k_{1},k_{2})$. The only thing
that is left to prove, therefore, is that the entries that we query
for $\mathfrak{g},\mathfrak{h}$ are also consistent. However, since
we access these input function using the same count combination $(k_{1},k_{2})$,
which is a consistent CC resulting from $x,y,a,b,c$, this must be
the case per the assumption stated in the Lemma. \end{proof}

\begin{lem}\label{lem:rr-ve-red} When the input functions are correctly defined on their
consistent entries, $\Red$ is correct.\end{lem}
\begin{proof}In Appendix~\ref{app:a}.\end{proof}

\begin{thm}
    RR-VE is correct, i.e., it
arrives at the identical solution as VE using the full tabular representation
of intermediate functions.  
\end{thm}
\begin{proof}We have shown that $\Aug$ is correct for consistent
entries in $\mathfrak{f}$ (Lemma~\ref{lem:rr-ve-aug}). The VE algorithm passes the result of
$\Aug$ to $\Red$ (see Figure \ref{fig:alg-ve}). In any of the cases
implemented by $\Red$, only consistent count combinations  for counters
$\#_{1}$ and $\#_{2}$ are computed. It follows that the max operation
will be over consistent entries in $\mathfrak{g}$, which are correct.
Since there are no other modifications to the regular VE algorithm,
RR-VE is correct.\end{proof}

\section{Exploiting Anonymity in the ALP}
\label{sec:5}
The results for RR-VE can be exploited in the ALP solution method that was introduced in Section~\ref{sec:effsoln}.
The non-linear max constraint in Equation~\ref{eqn:guestrin} is defined over functions $c_i \triangleq \gamma g_i-h_i\; \forall h_i\in H$ and reward factors $R_j, j=1,\ldots,r$, which are all locally-scoped and together span a factor graph.  As outlined previously, a VE procedure over this factor graph can translate the non-linear constraint into a set of linear constraints that is reduced compared to the standard formulation of the ALP.

The key insight of this section is that for a class of factored (M)MDPs defined with count aggregator functions in the 2TBN (i.e., precisely those that support \emph{anonymous influence}), 
the same intuition about reduced representations as in the previous section applies to implement the non-linear max constraint even more compactly.  We showed previously for the running example of disease control how anonymity naturally arises in the transition model of the stochastic process.

We first establish that basis functions $h_i\in H$, when back-projected through the 2TBN (which now includes mixed-mode functions), retain correct basis back-projections $g_i$ with reduced representations.  
The basis back-projection is computed with summation and product operations only (Equation~\ref{eqn:backprop}).  We have previously shown that summation (\Aug) of mixed-mode functions is correct for its consistent entries.  The same result holds for multiplication when replacing the sum operation with a multiplication.  It follows that $g_i$ (and $c_i$) share the compact reduced representations derived in Section~\ref{sec:3} and that they are correctly defined on their consistent entries.

The exact implementation of the max constraint in Equation~\ref{eqn:guestrin} with RR-VE proceeds as for the regular VE case.  All correctness results for RR-VE apply during the computation of the constraints in the RR-ALP.  
The number of variables and constraints is exponential only in the \emph{size of the representation} of the largest mixed-mode function formed during RR-VE.  Further, the representation with the smaller set of constraints is exact and yields the identical value function solution as the ALP that does not exploit anonymous influence.

\section{Experimental Evaluation}
We evaluate our methods on undirected disease propagation graphs with 30 and 50 nodes.  For the first round of experiments, we 
contrast runtimes of the normal VE/ALP method (where possible) with those that exploit ``anonymous influence'' in the graph.  Since the obtained value functions for both methods are identical, the focus of this evaluation is on runtime performance over a sampled set of random graphs.  
We then consider a disease control problem with 25 agents in a densely connected 50-node graph that cannot be solved with the normal ALP.  Problems of this size ($|\mathcal{S}|=2^{50}$, $|\mathcal{A}|=2^{25}$) are prohibitively large for exact solution methods to apply and are commonly solved heuristically.  To assess quality of the RR-ALP solution, we evaluate its policy performance against a 
vaccination heuristic in simulation. 

In all experiments, we use indicator functions $I_{X_i}$, $I_{\bar{X}_i}$ on each state variable (covering the two valid instantiations $\{\text{healthy},\text{infected}\}$) as the basis set $H$ in the (RR-)ALP.  We use identical transmission and node recovery rates throughout the graph, $\beta=0.6$, $\delta=0.3$.  Action costs are set to $\lambda_1=1$ and infection costs to $\lambda_2=50$.  All experiments use the identical greedy elimination heuristic for both VE and RR-VE, which minimizes the scope size of intermediate terms at the next iteration.

\subsubsection{Runtime Comparison}
We use graph-tool \cite{PeixotoGT2014} to generate 10 random graphs with an out-degree $k$ sampled from $P(k)\propto 1/k$, $k\in [1,10]$.  Out-degrees per node thus vary in $[1,10]$; the mean out-degree in the graphs in the test set ranges from 2.8 (graph 1) to 4.2 (graph 10).  Figure~\ref{fig:graphs} illustrates a subset of the resulting networks.

The runtime results comparing the VE/ALP method to RR-VE/RR-ALP are summarized in Table~\ref{tbl:resfirst}.  Shown are the number of constraints for each method, the wall-clock times for VE to generate the constraints, and the ALP runtimes to solve the value function after the constraints have been computed.
The last three columns show the relative magnitude of each measure, i.e. the gains in efficiency of the methods exploiting anonymous influence in each of the 10 random graphs.  On average, the RR-ALP solution time reduces to 16\% of the original ALP runtime while maintaining the identical solution.  Reductions by a factor of 50 are observed for two of the random graphs in the set (corresponding to the highlighted entries in the last column).
\begin{figure}[t]
  \centering
    \begin{tabular}{c@{\hskip 0.1cm}c@{\hskip 0.1cm}c@{\hskip 0.1cm}c}
      \includegraphics[width=0.18\textwidth]{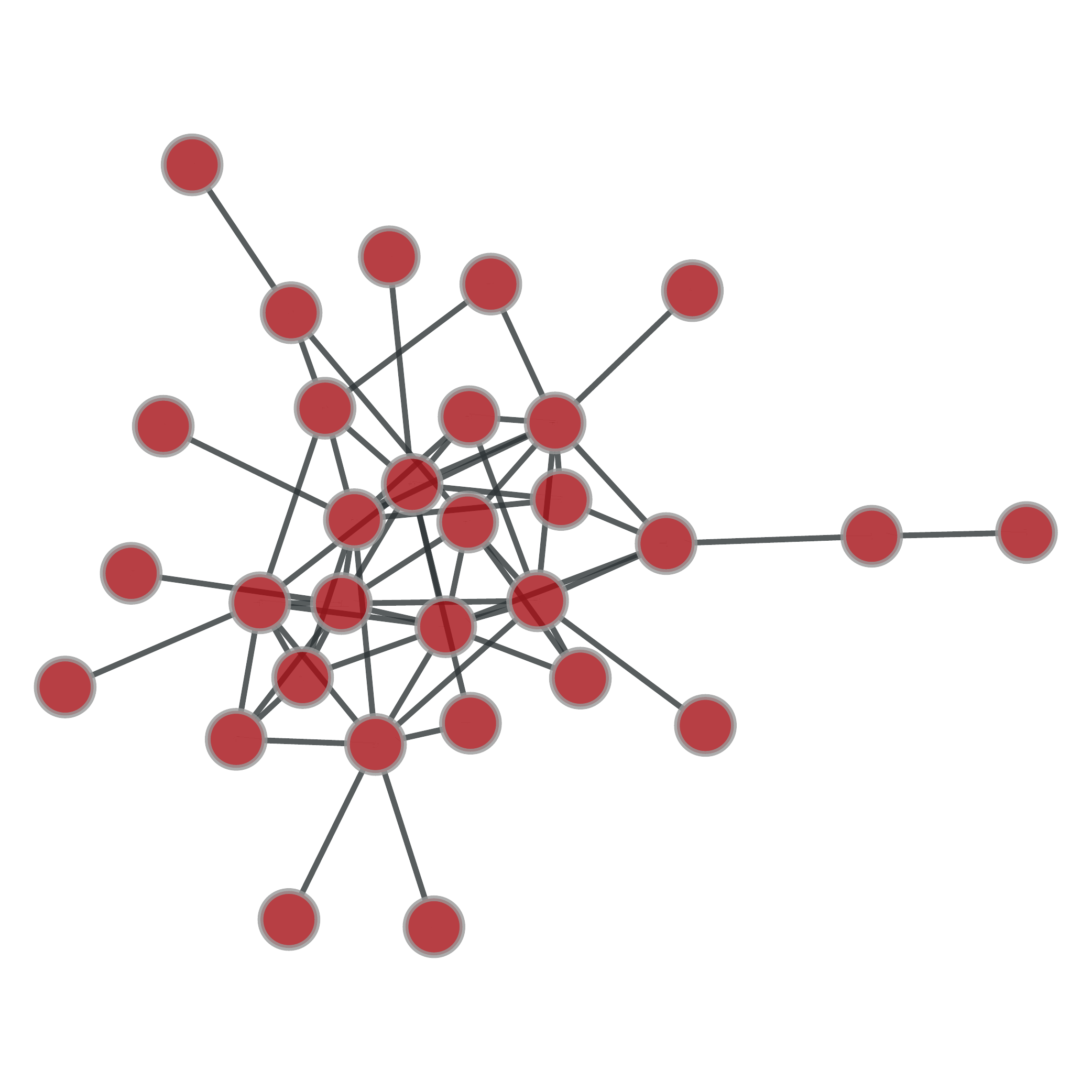} &
      \includegraphics[width=0.18\textwidth]{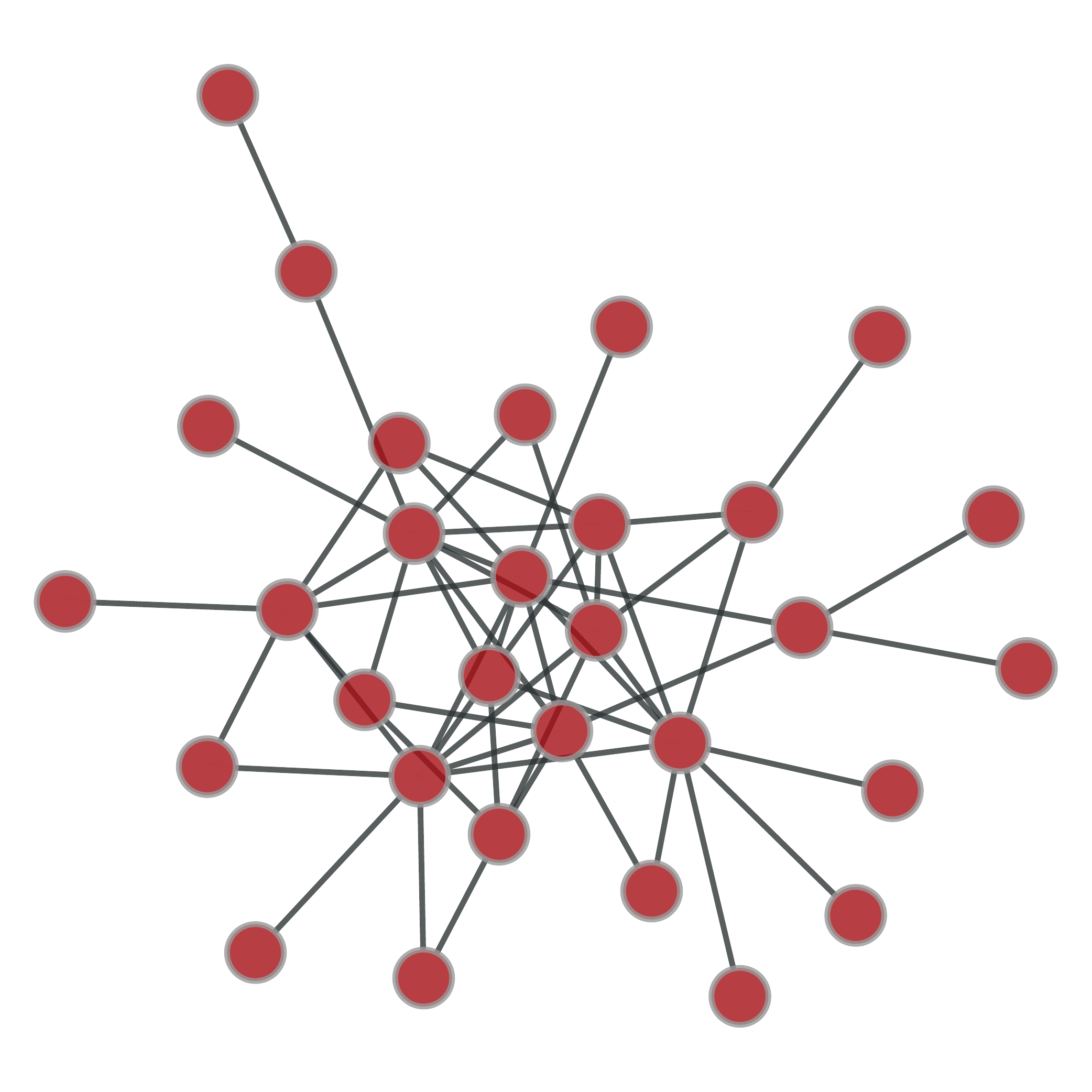} &
      \includegraphics[width=0.18\textwidth]{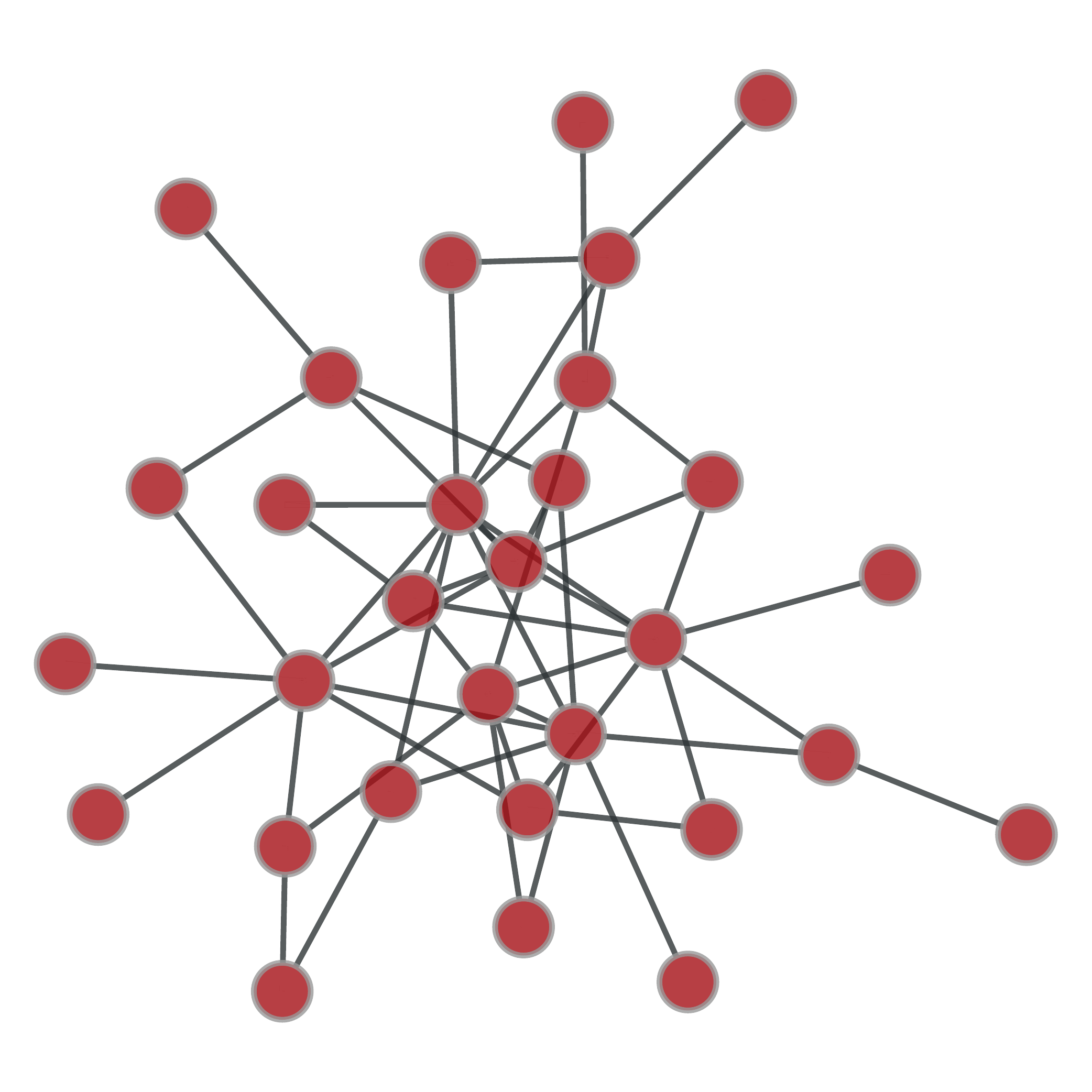} &
      \includegraphics[width=0.18\textwidth]{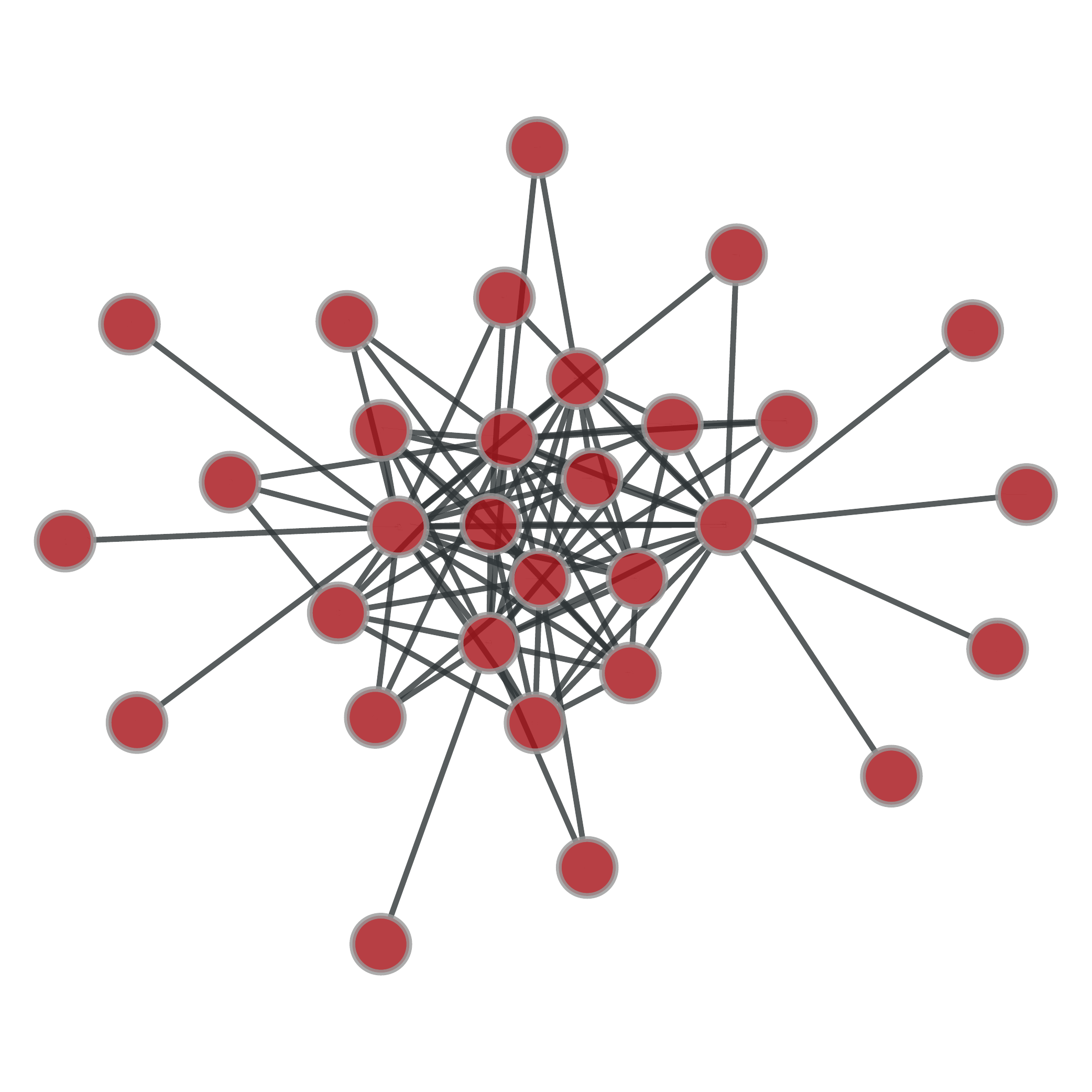} 
    \end{tabular}
  \caption{Sample of three random graphs in the test set with 30 nodes and a maximum out-degree of $10$.  Rightmost: test graph with increased out-degree sampled from $[1,20]$.}
  \label{fig:graphs}
\end{figure}
\begin{table}[!t]
\centering
\resizebox{0.77\columnwidth}{!}{%
\begin{tabular}{ |l l | c c c| }
\hline
$|\m{C1}|$, VE, ALP & $|\m{C}_{RR}|$, RR-VE, RR-ALP & $\dfrac{|\m{C}_{RR}|}{|\m{C1}|}$ & $\dfrac{\textnormal{RR-VE}}{\textnormal{VE}}$ & $\dfrac{\textnormal{RR-ALP}}{\textnormal{ALP}}$ \\
\hline
\hline
131475, 6.2s, \textbf{1085.8s} & 94023, 1.5s, \textbf{25.37s} & 0.72 &  0.24 & \textbf{0.02} \\
24595, 1.1s, 3.59s & 12515, 0.17s, 1.2s & 0.51 &  0.15 & 0.33 \\
55145, 3.5s, 30.43s & 27309, 0.4s, 8.63s & 0.5 &  0.11 & 0.28 \\
74735, 3.0s, 115.83s & 41711, 0.69s, 12.49s & 0.56 &  0.23 & 0.11 \\  
71067, 4.16s, 57.1s & 23619, 0.36s, 8.86s & 0.33 &  0.08 & 0.16 \\
\textbf{24615}, \textbf{1.6s}, 1.15s & \textbf{4539}, \textbf{0.07s}, 0.35s & \textbf{0.18} &  \textbf{0.04} & 0.30 \\
63307, 2.2s, 141.44s & 34523, 0.39s, 4.03s & 0.55 &  0.18 & 0.03 \\ 
57113, 0.91s, \textbf{123.16s} & 40497, 0.49s, \textbf{2.68s} & 0.71 &  0.54 & \textbf{0.02} \\
28755, 0.54s, 17.16 & 24819, 0.36s, 3.86s & 0.86 &  0.67 & 0.22 \\
100465, 2.47s, 284.75s & 38229, 0.62s, 36.76s & 0.38 &  0.25 & 0.13 \\ \hline\hline
\multicolumn{2}{|r|}{Average relative size:} & 0.53 & 0.25 & 0.16 \\
\hline
\end{tabular}%
}
\caption{Constraint set sizes, VE and ALP solution times for normal (column 1) and methods exploiting anonymous influence (column 2).  The last three columns show their relative magnitudes.  Maximal reductions are highlighted in bold.}
\label{tbl:resfirst}
\end{table}

We performed a final experiment with a graph with a larger out-degree ($k$ sampled from the interval $[1,20]$, shown at the right of Figure~\ref{fig:graphs}).  The disease propagation problem over this graph \emph{cannot be solved} with the normal VE/ALP because of exponential blow-up of intermediate terms.  The version exploiting anonymous influence completes successfully, performing constraint computation using RR-VE in 124.7s and generating $|\m{C}_{RR}|=5816731$ constraints.
\begin{figure*}
  \begin{center}
    \resizebox{1\columnwidth}{!}{%
    \begin{tabular}{c@{\hskip 1cm}c}
      \includegraphics{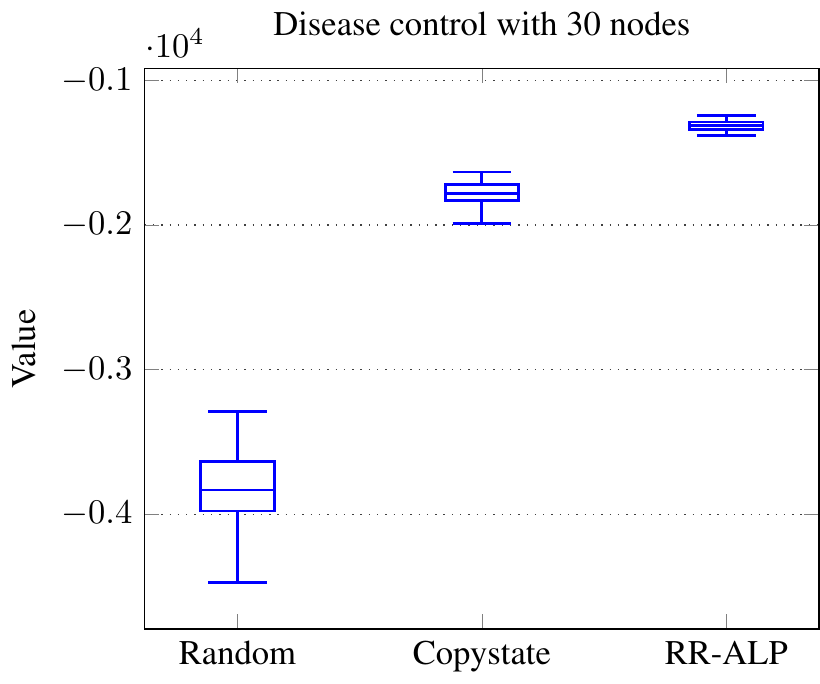} &
      \includegraphics{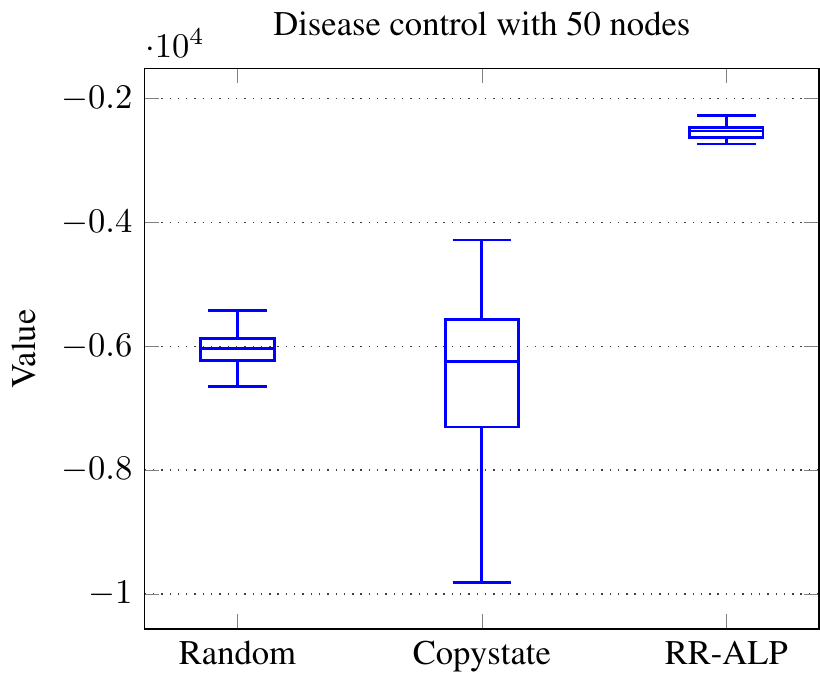}
    \end{tabular}
    }
  \end{center}
  \caption{Statistics of the mean returns of the three evaluated policies: \emph{random}, \emph{``copystate'' heuristic}, and \emph{RR-ALP policy} in the disease control problems.  Mean returns are computed over 50 randomly sampled starting states after 200 steps of policy simulation (each mean return is computed over 50 independent runs from a given $s_0$).  Visualized in the box plots are median, interquartile range ($IQR$), and \protect\raisebox{.2ex}{$\scriptstyle\pm$}$ 1.5\cdot IQR$ (upper and lower whiskers) of the mean returns.}
  \label{fig:dc30}
\end{figure*}

\subsubsection{Policy Performance}
In this section we show results of policy simulation for three distinct policies in the disease control task over two random graphs (30 nodes with 15 agents and 50 nodes with 25 agents, both with a maximum out-degree per node set to 15 neighbors).  The disease control problem over both graphs is infeasible for the regular VE/ALP due to the exponential increase of intermediate factors during VE.  We compare the solution of our RR-ALP method to a random policy and a heuristic policy that applies a vaccination action at $X_i$ if $X_i$ is infected in the current state and belongs to the controlled nodes $V_c$ in the graph.
The heuristic is reactive and does not provide anticipatory vaccinations if some of its parent nodes are infected; we refer to it as the ``copystate'' heuristic in our evaluation.  It serves as our main comparison metric for these large and densely connected graphs where optimal solutions are not available.  

To evaluate the policy performance, we compute the mean returns from 50 randomly sampled starting states $s_0$ after 200 steps of policy simulation (each mean return is computed over 50 independent runs from a given $s_0$).  Figure~\ref{fig:dc30} shows statistics of these mean returns, i.e., each sample underlying a box plot corresponds to a mean estimate from each of the 50 initial states.  The box plots thus
provide an indication of the sensitivity to the initial conditions $s_0$
in the disease graph.

The ``copystate'' heuristic works reasonably well in the 30-node/15-agent problem (left-hand side of Figure~\ref{fig:dc30}) but is consistently outperformed by the RR-ALP solution which can administer anticipatory vaccinations.  This effect actually becomes more pronounced with \emph{fewer} agents: we experimented with 6 agents in the identical graph and the results (not shown) indicate that the ``copystate'' heuristic performs significantly worse than the random policy.
This is presumably because blocking out disease paths early becomes more important with fewer agents since the lack of agents in other regions of the graph cannot make up for omissions later.  

In the 50-node/25-agent scenario the reactive ``copystate'' heuristic does not
provide a statistically significant improvement over a random policy (right-hand
side of Figure~\ref{fig:dc30}).  It is outperformed by the RR-ALP solution by
roughly a factor of 3 in our experiments.  In the same figure it is also apparent 
that the performance of the heuristic depends heavily on the initial state of the disease graph.  

Not shown in the Figure is that the RR-ALP policy also had a smaller variance \emph{within} all simulations from an individual initial state $s_0$, indicating not only better performance in expectation, but also higher reliability in each single case.

\section{Related Work}
\label{sec:7}
Many recent algorithms tackle domains with large (structured) state spaces.  For exact planning in factored domains, SPUDD exploits a decision diagram-based representation \cite{Hoey+al99:UAI}.  Monte Carlo tree search (MCTS) has been a popular online approximate planning method to scale to large domains \cite{Silver+al08:ICML}.  These methods do not apply to exponential action spaces without further approximations.  \citeauthor{Ho+al15:CDC} (\citeyear{Ho+al15:CDC}), for example, evaluated MCTS with three agents for a targeted version of the disease control problem.  Recent variants that exploit factorization \cite{Amato+Oliehoek15:AAAI} may be applicable.

Our work is based on earlier contributions of \citeauthor{Guestrin03:PhD} (\citeyear{Guestrin03:PhD}) on exploiting factored value functions to scale to large factored action spaces.  Similar assumptions can be exploited by inference-based approaches to planning which have been introduced for MASs where policies are represented as finite state controllers \cite{Kumar+al11:IJCAI}.  There are no assumptions about the policy in our approach.  The variational framework of \citeauthor{Cheng+al13:NIPS} (\citeyear{Cheng+al13:NIPS}) uses belief propagation (BP) and is exponential in the cluster size of the graph.  Their results are shown for 20-node graphs with out-degree 3 and a restricted class of chain graphs.  Our method remains exponential in tree-width but exploits anonymous influence in the graph to scale to random graphs with denser connectivity.

Generalized counts in first-order (FO) models eliminate indistiguishable variables in the same predicate in a single operation \cite{Sanner+Boutilier09:AI,Milch+al08:AAAI}.  Our contributions are distinct from FO methods.  Anonymous influence applies in propositional models and to node sets that are not necessarily indistiguishable in the problem.  
We also show that shattering into disjoint counter scopes is not required during VE and show how this results in efficiency gains during VE.  

There is a conceptual link to approaches that exploit anonymity or influence-based abstraction in decentralized or partially-observable frameworks.  \citeauthor{Oliehoek12AAAI_IBA} (\citeyear{Oliehoek12AAAI_IBA}) define influence-based policy abstraction for factored Dec-POMDPs, which formalizes how different \emph{policies} of other agents may lead to the same influence. Roughly stated, this work can be seen to give some justification of the idea of using factored value functions. While optimal influence search for TD-POMDPs \cite{Witwicki+Durfee10:ICAPS,Witwicki12AAMAS} only searches the space of unique influences (which implicitly does take into account the working of aggregation operators), such a procedure is not yet available for general factored Dec-POMDPs, and would require imposing decentralization constraints (i.e., restrictions on what state factors agents can base their actions on) for MMDPs. Our approach, in contrast, does not impose such constraints and provides a more scalable approach for MMDPs by introducing a practical way of dealing with aggregation operators.

Also closely related is the work by \citeauthor{Varakantham+al14:AAAI} (\citeyear{Varakantham+al14:AAAI}) on exploiting agent anonymity in transitions and rewards in a subclass of Dec-MDPs with specific algorithms to solve them.  Our definition of anonymity extends to both action and state variables; our results on compact, redundant representation of anonymous influence further also applies outside of planning (e.g., for efficient variable elimination).

\section{Conclusions and Future Work}
This paper introduces the concept of ``anonymous influence'' in large factored multiagent MDPs and shows how it can be exploited to scale variable elimination and approximate linear programming beyond what has been previously solvable.  
The key idea is that both representational and computational benefits follow from reasoning about influence of variable sets rather than variable identity in the factor graph.  These results hold for both single and multiagent factored MDPs and are exact reductions, yielding the identical result to the normal VE/ALP, while greatly extending the class of graphs that can be solved.  Potential future directions include approximate methods (such as loopy BP) in the factor graph to scale the ALP to even larger problems and to support increased basis function coverage in more complex graphs.

\section*{Acknowledgments}
F.O. is supported by NWO Innovational Research Incentives Scheme Veni \#639.021.336.

\bibliography{influence.arxiv.v2}

\begin{thebibliography}{}

\bibitem[\protect\citeauthoryear{Amato and
  Oliehoek}{2015}]{Amato+Oliehoek15:AAAI}
Amato, C., and Oliehoek, F.~A.
\newblock 2015.
\newblock Scalable planning and learning for multiagent {POMDPs}.
\newblock In {\em AAAI Conference on Artificial Intelligence (AAAI)},
  1995--2002.

\bibitem[\protect\citeauthoryear{Bailey}{1957}]{Bailey57:book}
Bailey, N. T.~J.
\newblock 1957.
\newblock {\em The Mathematical Theory of Epidemics}.
\newblock London: C. Griffin \& Co.

\bibitem[\protect\citeauthoryear{Bernstein \bgroup et al\mbox.\egroup
  }{2002}]{Bernstein+al02:OR}
Bernstein, D.~S.; Givan, R.; Immerman, N.; and Zilberstein, S.
\newblock 2002.
\newblock The complexity of decentralized control of {M}arkov decision
  processes.
\newblock {\em Mathematics of Operations Research} 27(4):819--840.

\bibitem[\protect\citeauthoryear{Boutilier, Dean, and
  Hanks}{1999}]{Boutilier+al99:JAIR}
Boutilier, C.; Dean, T.; and Hanks, S.
\newblock 1999.
\newblock Decision-theoretic planning: Structural assumptions and computational
  leverage.
\newblock {\em Journal of Artificial Intelligence Research} 11:1--94.

\bibitem[\protect\citeauthoryear{Boutilier}{1996}]{Boutilier96:TARK}
Boutilier, C.
\newblock 1996.
\newblock Planning, learning and coordination in multiagent decision processes.
\newblock In Shoham, Y., ed., {\em Theoretical Aspects of Rationality and
  Knowledge},  195--201.

\bibitem[\protect\citeauthoryear{Cheng \bgroup et al\mbox.\egroup
  }{2013}]{Cheng+al13:NIPS}
Cheng, Q.; Liu, Q.; Chen, F.; and Ihler, A.
\newblock 2013.
\newblock Variational planning for graph-based {MDPs}.
\newblock In {\em Advances in Neural Information Processing Systems (NIPS)},
  2976--2984.

\bibitem[\protect\citeauthoryear{Cornelius, Kath, and
  Motter}{2013}]{Cornelius+al13:NAT}
Cornelius, S.~P.; Kath, W.~L.; and Motter, A.~E.
\newblock 2013.
\newblock Realistic control of network dynamics.
\newblock {\em Nature Commun.} 4(1942):1--9.

\bibitem[\protect\citeauthoryear{Cui \bgroup et al\mbox.\egroup
  }{2015}]{Cui+al15:AAAI}
Cui, H.; Khardon, R.; Fern, A.; and Tadepalli, P.
\newblock 2015.
\newblock Factored {MCTS} for large scale stochastic planning.
\newblock In {\em AAAI Conference on Artificial Intelligence (AAAI)},
  3261--3267.

\bibitem[\protect\citeauthoryear{Dechter}{2013}]{Dechter13:book}
Dechter, R.
\newblock 2013.
\newblock Reasoning with probabilistic and deterministic graphical models:
  Exact algorithms.
\newblock {\em Synthesis Lectures on Artificial Intelligence and Machine
  Learning} 7(3):1--191.

\bibitem[\protect\citeauthoryear{Guestrin \bgroup et al\mbox.\egroup
  }{2003}]{Guestrin+al03:JAIR}
Guestrin, C.; Koller, D.; Parr, R.; and Venkataraman, S.
\newblock 2003.
\newblock Efficient solution algorithms for factored {MDP}s.
\newblock {\em Journal of Artificial Intelligence Research} 19:399--468.

\bibitem[\protect\citeauthoryear{Guestrin, Koller, and
  Parr}{2002}]{Guestrin+al02:NIPS}
Guestrin, C.; Koller, D.; and Parr, R.
\newblock 2002.
\newblock Multiagent planning with factored {MDP}s.
\newblock In {\em Advances in Neural Information Processing Systems (NIPS)},
  1523--1530.

\bibitem[\protect\citeauthoryear{Guestrin}{2003}]{Guestrin03:PhD}
Guestrin, C.
\newblock 2003.
\newblock {\em Planning Under Uncertainty in Complex Structured Environments}.
\newblock Ph.D. Dissertation, Computer Science Department, Stanford University.

\bibitem[\protect\citeauthoryear{Ho \bgroup et al\mbox.\egroup
  }{2015}]{Ho+al15:CDC}
Ho, C.; Kochenderfer, M.~J.; Mehta, V.; and Caceres, R.~S.
\newblock 2015.
\newblock Control of epidemics on graphs.
\newblock In {\em IEEE Conference on Decision and Control (CDC)}.

\bibitem[\protect\citeauthoryear{Hoey \bgroup et al\mbox.\egroup
  }{1999}]{Hoey+al99:UAI}
Hoey, J.; St-Aubin, R.; Hu, A.~J.; and Boutilier, C.
\newblock 1999.
\newblock {SPUDD}: Stochastic planning using decision diagrams.
\newblock In {\em Conference on Uncertainty in Artificial Intelligence (UAI)}.

\bibitem[\protect\citeauthoryear{Kochenderfer}{2015}]{Kochenderfer15:DMU}
Kochenderfer, M.~J.
\newblock 2015.
\newblock {\em Decision Making Under Uncertainty: Theory and Application}.
\newblock MIT Press.

\bibitem[\protect\citeauthoryear{Kok and Vlassis}{2006}]{Kok+Vlassis06:JMLR}
Kok, J.~R., and Vlassis, N.~A.
\newblock 2006.
\newblock Collaborative multiagent reinforcement learning by payoff
  propagation.
\newblock {\em Journal of Machine Learning Research} 7:1789--1828.

\bibitem[\protect\citeauthoryear{Koller and
  Friedman}{2009}]{Koller+Friedman09:PGM}
Koller, D., and Friedman, N.
\newblock 2009.
\newblock {\em Probabilistic Graphical Models: Principles and Techniques}.
\newblock MIT Press.

\bibitem[\protect\citeauthoryear{Koller and Parr}{1999}]{Koller+Parr99:IJCAI}
Koller, D., and Parr, R.
\newblock 1999.
\newblock Computing factored value functions for policies in structured {MDP}s.
\newblock In {\em International Joint Conference on Artificial Intelligence
  (IJCAI)},  1332--1339.

\bibitem[\protect\citeauthoryear{Kumar, Zilberstein, and
  Toussaint}{2011}]{Kumar+al11:IJCAI}
Kumar, A.; Zilberstein, S.; and Toussaint, M.
\newblock 2011.
\newblock Scalable multiagent planning using probabilistic inference.
\newblock In {\em International Joint Conference on Artificial Intelligence
  (IJCAI)},  2140--2146.

\bibitem[\protect\citeauthoryear{Liu, Slotine, and
  Barabasi}{2011}]{Liu+al11:NAT}
Liu, Y.-Y.; Slotine, J.-J.; and Barabasi, A.-L.
\newblock 2011.
\newblock {Controllability of complex networks}.
\newblock {\em Nature} 473(7346):167--173.

\bibitem[\protect\citeauthoryear{Milch \bgroup et al\mbox.\egroup
  }{2008}]{Milch+al08:AAAI}
Milch, B.; Zettlemoyer, L.~S.; Kersting, K.; Haimes, M.; and Kaelbling, L.~P.
\newblock 2008.
\newblock Lifted probabilistic inference with counting formulas.
\newblock In {\em AAAI Conference on Artificial Intelligence (AAAI)},
  1062--1068.

\bibitem[\protect\citeauthoryear{Nowzari, Preciado, and
  Pappas}{2015}]{Nowzari+al15:arXiv}
Nowzari, C.; Preciado, V.~M.; and Pappas, G.~J.
\newblock 2015.
\newblock Analysis and control of epidemics: A survey of spreading processes on
  complex networks.
\newblock Technical Report arXiv:1505.00768.

\bibitem[\protect\citeauthoryear{Oliehoek, Whiteson, and
  Spaan}{2013}]{Oliehoek13AAMAS}
Oliehoek, F.~A.; Whiteson, S.; and Spaan, M. T.~J.
\newblock 2013.
\newblock Approximate solutions for factored {Dec-POMDPs} with many agents.
\newblock In {\em International Conference on Autonomous Agents and Multiagent
  Systems (AAMAS)},  563--570.

\bibitem[\protect\citeauthoryear{Oliehoek, Witwicki, and
  Kaelbling}{2012}]{Oliehoek12AAAI_IBA}
Oliehoek, F.~A.; Witwicki, S.; and Kaelbling, L.~P.
\newblock 2012.
\newblock Influence-based abstraction for multiagent systems.
\newblock In {\em AAAI Conference on Artificial Intelligence (AAAI)},
  1422--1428.

\bibitem[\protect\citeauthoryear{Peixoto}{2014}]{PeixotoGT2014}
Peixoto, T.~P.
\newblock 2014.
\newblock The graph-tool python library.
\newblock {\em figshare}.
\newblock DOI: 10.6084/m9.figshare.1164194.

\bibitem[\protect\citeauthoryear{Puterman}{2005}]{Puterman05:MDP}
Puterman, M.~L.
\newblock 2005.
\newblock {\em Markov Decision Processes: Discrete Stochastic Dynamic
  Programming}.
\newblock New York: Wiley.

\bibitem[\protect\citeauthoryear{Raghavan \bgroup et al\mbox.\egroup
  }{2012}]{Raghavan+al12:AAAI}
Raghavan, A.; Joshi, S.; Fern, A.; Tadepalli, P.; and Khardon, R.
\newblock 2012.
\newblock Planning in factored action spaces with symbolic dynamic programming.
\newblock In {\em AAAI Conference on Artificial Intelligence (AAAI)}.

\bibitem[\protect\citeauthoryear{Sanner and
  Boutilier}{2009}]{Sanner+Boutilier09:AI}
Sanner, S., and Boutilier, C.
\newblock 2009.
\newblock Practical solution techniques for first-order {MDPs}.
\newblock {\em Artificial Intelligence} 173(5-6):748--788.

\bibitem[\protect\citeauthoryear{Silver, Sutton, and
  M{\"u}ller}{2008}]{Silver+al08:ICML}
Silver, D.; Sutton, R.~S.; and M{\"u}ller, M.
\newblock 2008.
\newblock Sample-based learning and search with permanent and transient
  memories.
\newblock In {\em International Conference on Machine Learning (ICML)},
  968--975.

\bibitem[\protect\citeauthoryear{Taghipour \bgroup et al\mbox.\egroup
  }{2013}]{Taghipour+alJAIR13}
Taghipour, N.; Fierens, D.; Davis, J.; and Blockeel, H.
\newblock 2013.
\newblock {Lifted variable elimination: Decoupling the operators from the
  constraint language}.
\newblock {\em Journal of Artificial Intelligence Research} 47:393--439.

\bibitem[\protect\citeauthoryear{Varakantham, Adulyasak, and
  Jaillet}{2014}]{Varakantham+al14:AAAI}
Varakantham, P.; Adulyasak, Y.; and Jaillet, P.
\newblock 2014.
\newblock Decentralized stochastic planning with anonymity in interactions.
\newblock In {\em AAAI Conference on Artificial Intelligence (AAAI)},
  2505--2512.

\bibitem[\protect\citeauthoryear{Witwicki and
  Durfee}{2010}]{Witwicki+Durfee10:ICAPS}
Witwicki, S.~J., and Durfee, E.~H.
\newblock 2010.
\newblock Influence-based policy abstraction for weakly-coupled {D}ec-{POMDP}s.
\newblock In {\em International Conference on Automated Planning and Scheduling
  (ICAPS)},  185--192.

\bibitem[\protect\citeauthoryear{Witwicki, Oliehoek, and
  Kaelbling}{2012}]{Witwicki12AAMAS}
Witwicki, S.; Oliehoek, F.~A.; and Kaelbling, L.~P.
\newblock 2012.
\newblock Heuristic search of multiagent influence space.
\newblock In {\em International Conference on Autonomous Agents and Multiagent
  Systems (AAMAS)},  973--981.

\end{thebibliography}
\bibliographystyle{aaai}

\appendix

\section{Correctness of Redundant Representation VE}
\label{app:a}

Here we prove Lemma~\ref{lem:rr-ve-red} that is used to show correctness of Redundant Representation VE (RR-VE) in Section~\ref{sec:4}.

\paragraph{Correctness of Reduce}

\begin{lem}When the input functions are correctly defined on their
consistent entries, $\Red$ is correct.\end{lem}

\begin{proof}We need to show that for a function $g(x,y,z,\#_{1}(a,b,z),\#_{2}(b,c))$
if we reduce by maxing out any variable, we indeed get the desired
function $f$. Again, the resulting representation might contain inconsistent
entries, but we only need to show correctness for the consistent entries,
since only those will be queried. This is because the result of $\Red$
only occurs as input to $\Aug$ in Figure \ref{fig:alg-ve} and we
have shown previously that only consistent entries are queried from
these input functions.

We discriminate the different cases:

\subparagraph{Maxing out a Proper Variable }

If we max out $x$, our redundant representation performs the
operation
\begin{equation}
\mathfrak{f}(y,z,k_{1},k_{2})\triangleq\max\left\{ \mathfrak{g}(0,y,z,k_{1},k_{2}),\mathfrak{g}(1,y,z,k_{1},k_{2})\right\} \label{eq:RR-Reduce-propervar-1}
\end{equation}

Assume an arbitrary $y,z,a,b,c$. We need to show that $\mathfrak{f}$
represents function $f$ correctly, i.e. 
\[
f(y,z,\#_{1}(a,b,z),\#_{2}(b,c))=\max_{x\in\{0,1\}}g(x,y,z,\#_{1}(a,b,z),\#_{2}(b,c))
\]

We start with the r.h.s.:
\begin{multline}
\max_{x\in\{0,1\}}g(x,y,z,\#_{1}(a,b,z),\#_{2}(b,c))=\\
\max\left\{ g(0,y,z,\#_{1}(a,b,z),\#_{2}(b,c)),g(1,y,z,\#_{1}(a,b,z),\#_{2}(b,c))\right\} \label{eq:proof-shared-counter-blkafblk-1-1-1-1}
\end{multline}

Suppose $\#_{1}(a,b,z)=k_{1}$ and $\#_{2}(b,c)=k_{2}$. Then \eqref{eq:proof-shared-counter-blkafblk-1-1-1-1}
is equal to 
\[
\max\left\{ \mathfrak{g}(0,y,z,k_{1},k_{2}),\mathfrak{g}(1,y,z,k_{1},k_{2})\right\} 
\]
but this is exactly how $\mathfrak{f}(x,y,z,k_{1},k_{2})$ is defined,
thereby showing that this representation of $f$ is correct provided
that the accessed entries for $\mathfrak{g}$ are correct. But $\mathfrak{g}$
is only accessed on $\left(k_{1},k_{2}\right)$, which is a consistent
entry that we assumed to be correct. Realizing that we showed correctness
for the arbitrarily selected $y,z,a,b,c$, and hence for all $y,z,a,b,c$,
we complete the proof.

\subparagraph{Maxing out a Non-Shared Count Variable}

If we max out $a$, our redundant representation performs the 
operation
\begin{equation}
\mathfrak{f}(x,y,z,k_{1},k_{2})\triangleq\max\left\{ \mathfrak{g}(x,y,z,k_{1},k_{2}),\mathfrak{g}(x,y,z,k_{1}+1,k_{2})\right\} \label{eq:RR-Reduce-non-shared-count-1}
\end{equation}

Assume an arbitrary $x,y,z,b,c$. We need to show that $\mathfrak{f}$
represents function $f$ correctly, i.e. 
\[
f(x,y,z,\#_{1}'(b,z),\#_{2}(b,c))=\max_{a\in\{0,1\}}g(x,y,z,\#_{1}(a,b,z),\#_{2}(b,c)),
\]
with $\#_{1}'(b,z)$ being a \emph{reduced }CA.

We start with the r.h.s.:
\begin{multline}
\max_{a\in\{0,1\}}g(x,y,z,\#_{1}(a,b,z),\#_{2}(b,c))=\\
\max\left\{ g(x,y,z,\#_{1}(0,b,z),\#_{2}(b,c)),g(x,y,z,\#_{1}(1,b,z),\#_{2}(b,c))\right\} \label{eq:proof-shared-counter-blkafblk-1-1-1}
\end{multline}

Suppose $\#_{1}(0,b,z)=k_{1}$, $\#_{2}(b,c)=k_{2}$ then $\#_{1}(1,b,z)=k_{1}+1$
such that \eqref{eq:proof-shared-counter-blkafblk-1-1-1} is equal
to 
\[
\max\left\{ \mathfrak{g}(x,y,z,k_{1},k_{2}),\mathfrak{g}(x,y,z,k_{1}+1,k_{2})\right\} 
\]
but this is exactly how $\mathfrak{f}(x,y,z,k_{1},k_{2})$ is defined,
thereby showing that this representation of $f$ is correct for the
arbitrarily selected $x,y,z,b,c$, and hence for all $x,y,z,b,c$,
provided that $\mathfrak{g}(x,y,z,k_{1},k_{2})$ and $\mathfrak{g}(x,y,z,k_{1}+1,k_{2})$
are computed correctly. But both $(k_{1},k_{2})$ and $(k_{1}+1,k_{2})$
results from settings of particular values for $a,b,c,z$ and thus
are consistent CCs. Since we assumed that the input functions are
correct on the consistent entiries, the computation of $\mathfrak{f}$
is correct.

\subparagraph{Maxing out a Shared Counter Variable}

We are given $g(x,y,z,\#_{1}(a,b,z),\#_{2}(b,c))$ which is represented
as $\mathfrak{g}(x,y,z,k_{1},k_{2})$. If we max out $b$, our redundant
representation performs the following operation.
\begin{equation}
\mathfrak{f}(x,y,z,k_{1},k_{2})\triangleq\max\left\{ \mathfrak{g}(x,y,z,k_{1},k_{2}),\mathfrak{g}(x,y,z,k_{1}+1,k_{2}+1)\right\} \label{eq:RR-Reduce-shared-count-1}
\end{equation}
Here we want to show that this leads to the correct result. That is,
we need to prove that the represented functions represent the right
thing:
\[
\forall_{x,y,z,a,c}\qquad f(x,y,z,\#_{1}'(a,z),\#_{2}'(c))=\max_{b\in\{0,1\}}g(x,y,z,\#_{1}(a,b,z),\#_{2}(b,c)),
\]
with $\#_{1}'(a,z),\#_{2}'(c)$ being \emph{reduced }CAs.

Assume an arbitrary $x,y,z,a,c$. We need to show that 
\[
f(x,y,z,\#_{1}'(a,z),\#_{2}'(c))=\max_{b\in\{0,1\}}g(x,y,z,\#_{1}(a,b,z),\#_{2}(b,c))
\]
We start with the r.h.s.:
\begin{multline}
\max_{b\in\{0,1\}}g(x,y,z,\#_{1}(a,b,z),\#_{2}(b,c))=\\
\max\left\{ g(x,y,z,\#_{1}(a,0,z),\#_{2}(0,c)),g(x,y,z,\#_{1}(a,1,z),\#_{2}(1,c))\right\} \label{eq:proof-shared-counter-blkafblk-1}
\end{multline}

Suppose$\#_{1}(a,0,z)=k_{1}$, $\#_{2}(0,c)=k_{2}$ then $\#_{1}(a,1,z)=k_{1}+1$
and $\#_{2}(1,c)=k_{2}+1$, such that \eqref{eq:proof-shared-counter-blkafblk-1}
is equal to 
\[
\max\left\{ \mathfrak{g}(x,y,z,k_{1},k_{2}),\mathfrak{g}(x,y,z,k_{1}+1,k_{2}+1)\right\} 
\]
but this is exactly how $\mathfrak{f}(x,y,z,k_{1},k_{2})$ is defined,
thereby showing that this representation of $f$ is correct for the
arbitrarily selected $x,y,z,a,c$, and hence for all $x,y,z,a,c$,
provided that $\mathfrak{g}(x,y,z,k_{1},k_{2})$ and $\mathfrak{g}(x,y,z,k_{1}+1,k_{2}+1)$
are correct. Again, these are consistent entries,  thus completing
the proof.

\subparagraph{Maxing out a Shared Proper/Counter Variable}

In case we max out $z$, our redundant representation performs the operation
\begin{equation}
\mathfrak{f}(x,y,k_{1},k_{2})\triangleq\max\left\{ \mathfrak{g}(x,y,0,k_{1},k_{2}),\mathfrak{g}(x,y,1,k_{1}+1,k_{2})\right\} \label{eq:RR-Reduce-shared-proper-count-1}
\end{equation}

Assume an arbitrary $x,y,a,b,c$. We need to show that $\mathfrak{f}$
represents function $f$ correctly, i.e. 
\[
f(x,y,\#_{1}'(a,b),\#_{2}(b,c))=\max_{z\in\{0,1\}}g(x,y,z,\#_{1}(a,b,z),\#_{2}(b,c)),
\]
with $\#_{1}'(a,b)$ being a \emph{reduced }CA.

We start with the r.h.s.:
\begin{multline}
\max_{z\in\{0,1\}}g(x,y,z,\#_{1}(a,b,z),\#_{2}(b,c))=\\
\max\left\{ g(x,y,0,\#_{1}(a,b,0),\#_{2}(b,c)),g(x,y,1,\#_{1}(a,b,1),\#_{2}(b,c))\right\} \label{eq:proof-shared-counter-blkafblk-1-1}
\end{multline}

Suppose $\#_{1}(a,b,0)=k_{1}$, $\#_{2}(b,c)=k_{2}$. Then $\#_{1}(a,b,1)=k_{1}+1$,
such that \eqref{eq:proof-shared-counter-blkafblk-1-1} is equal to
\[
\max\left\{ \mathfrak{g}(x,y,0,k_{1},k_{2}),\mathfrak{g}(x,y,1,k_{1}+1,k_{2})\right\} 
\]
but this is exactly how $\mathfrak{f}(x,y,z,k_{1},k_{2})$ is defined.
Since the maximization is over consistent entries, this shows that
this representation of $f$ is correct for the arbitrarily selected
$x,y,a,b,c$, and hence for all $x,y,a,b,c$, thus completing the
proof. \end{proof}

\end{document}